\documentclass[11pt]{amsart}

\usepackage[utf8]{inputenc}
\usepackage[T1]{fontenc}
\usepackage[a4paper,margin=1in]{geometry}
\usepackage[english]{babel}

\usepackage{amsmath,amsfonts,amssymb,amsthm,mathtools}
\usepackage{bm}
\usepackage{mathrsfs}

\usepackage{graphicx}
\usepackage{float}
\usepackage{booktabs}
\usepackage{multirow}
\usepackage{algorithm}
\usepackage{algorithmic}

\usepackage{url}
\usepackage{csquotes}
\usepackage{xcolor}
\usepackage[colorlinks=true,linkcolor=blue,citecolor=blue,urlcolor=blue,unicode]{hyperref}

\newtheorem{theorem}{Theorem}[section]
\newtheorem{definition}{Definition}[section]
\newtheorem{proposition}{Proposition}[section]
\newtheorem{corollary}{Corollary}[section]

\newtheorem{lemma}{Lemma}

\newtheorem*{remark}{Remark}

\def\to{\rightarrow}

\def\bal{\begin{aligned}}
	\def\ea{\end{aligned}}

\def\bdf{\begin{definition}}
	\def\edf{\end{definition}}
\def\bth{\begin{theorem}}
	\def\eth{\end{theorem}}
\def\bprp{\begin{proposition}}
	\def\eprp{\end{proposition}}
\def\bclr{\begin{corollary}}
	\def\eclr{\end{corollary}}
\def\blm{\begin{lemma}}
	\def\elm{\end{lemma}}
\def\brk{\begin{remark}}
	\def\erk{\end{remark}}

\def\bmtx{\left\{\begin{array}}
	\def\emtx{\end{array}\right\}}
\def\eeqts{\end{array}\right.}

\newcommand{\SPD}{\operatorname{SPD}}
\newcommand{\Sym}{\operatorname{Sym}}
\newcommand{\Orth}{\operatorname{O}}
\newcommand{\Logm}{\operatorname{Log}}

\newcommand{\Id}{\operatorname{Id}}
\newcommand{\RicD}{\operatorname{Ric}^{\mathrm D}}
\newcommand{\RicDiv}{\operatorname{Ric}^{\mathrm{D,div}}}

\newcommand{\dist}{\operatorname{d}}
\newcommand{\cC}{\mathcal C}

\title[Holonomy-Based Curvature Discretization]{Geometric Structures on Graphs:\\
	A Holonomy-Based Discretization of Curvature}
\author[H. Li]{Hao Li\textsuperscript{1}}
\author[Y. Peng]{Yuhan Peng\textsuperscript{2}}
\author[J. Dong]{Junwen Dong\textsuperscript{3}}

\thanks{%
\textsuperscript{1}Department of Mathematics, National University of Singapore, Singapore. Email: \texttt{E1311647@u.nus.edu}.\protect\\
\textsuperscript{2}School of Physical and Mathematical Sciences, Nanyang Technological University, Singapore. Email: \texttt{YUHAN018@e.ntu.edu.sg}.\protect\\
\textsuperscript{3}Chern Institute of Mathematics, Nankai University, Tianjin, China. Email: \texttt{N2505367A@e.ntu.edu.sg}.%
}

\date{\today}

\usepackage{xparse}
\makeatletter
\newcommand{\runningtitle}{\MakeUppercase{\shorttitle}}
\let\oldsection\section
\RenewDocumentCommand{\section}{s o m}{%
	\IfBooleanTF{#1}
	{\oldsection*{#3}}
	{%
		\IfNoValueTF{#2}
		{\oldsection{#3}\markboth{\MakeUppercase{\thesection.\ #3}}{\runningtitle}}
		{\oldsection[#2]{#3}\markboth{\MakeUppercase{\thesection.\ #2}}{\runningtitle}}%
	}%
}
\makeatother

\begin{document}
	
\begin{abstract}
	We propose a holonomy-based framework for discretizing curvature on graphs equipped with local symmetric positive-definite metrics.  Each vertex carries a fibre metric \(g_i\), and each directed edge carries a reversible metric-compatible transport \(F_{ij}\).  The ordered product around an oriented triangular loop \(\mathcal C\) gives a holonomy \(H_{\mathcal C}\), whose normalized logarithm \(\Omega_{\mathcal C}=-s_{\mathcal C}^{-1}\operatorname{Log}(H_{\mathcal C})\) is used as a finite-loop curvature observation.  Thus the construction discretizes the geometric principle that infinitesimal holonomy is controlled by curvature, rather than treating holonomy as a heuristic feature.  Since \(\Omega_{\mathcal C}\) lies in the \(g_i\)-orthogonal Lie algebra, it is not itself a velocity of an SPD metric.  We therefore introduce two aggregation mechanisms: a commutator with a symmetric response matrix, producing symmetric Ricci-type metric responses, and an incidence-aware covariant divergence of curvature-induced edge fluxes, reflecting the relation between trace and covariant divergence.  The resulting responses are locally orthogonal-gauge equivariant and can drive exponential updates that preserve positive definiteness.  We also give a reversible metric-compatible parametrization of edge transports, allowing orthogonal edge factors, loop scales, weights, and response matrices to be learned while respecting the graph geometry.  Known-geometry calibrations on the unit sphere test the holonomy--curvature relation, curvature preservation under nontrivial local metric representations, and the empirical recovery of edge transports from local observations.
\end{abstract}
	
	\maketitle
	
	\section{Introduction}
	
Curvature has provided an important geometric language for network analysis and graph learning.  Ollivier--Ricci curvature compares neighbourhood probability measures by optimal transport \cite{Ollivier2009}, while Forman--Ricci curvature arises from a Bochner--Weitzenbock type principle on weighted cell complexes \cite{Forman2003}.  Under suitable sampling and scaling assumptions, rescaled Ollivier--Ricci curvature on random geometric graphs can also converge to the Ricci curvature of an underlying Riemannian manifold \cite{vanderHoorn2023}.  These notions and their variants have been used for community detection and graph Ricci flow \cite{Ni2019}, over-squashing and curvature-based rewiring \cite{Topping2022}, and curvature-weighted message passing in graph learning models \cite{Li2022}.
	
	In many graph-learning uses, however, curvature enters the model after being compressed into a scalar descriptor on edges or vertices, where it modulates neighborhood weights, rewiring rules, or information-propagation strength.  This effectively exploits local structural information revealed by curvature, but it does not directly discretize the tensor algebra or geometric construction mechanism of curvature itself.  This leads to the question addressed in this paper: instead of starting from a prescribed scalar curvature, can graph interactions themselves be organized as discrete geometric data with local metrics and edgewise parallel transports, so that curvature is obtained from their failure to compose trivially around local closed loops?
	
	We study this question through the relation between local metric structures, edgewise parallel transport, and holonomy on graphs.  Each vertex \(i\) is assigned a fibre \(E_i\cong\mathbb R^n\) with a positive-definite metric \(g_i\in\SPD_n\), and a linear map on a directed edge \(j\to i\),
	\[
	F_{ij}:E_j\to E_i,
	\]
	is interpreted as a discrete parallel transport from the local fibre at \(j\) to the local fibre at \(i\).  Metric compatibility requires
	\[
	F_{ij}^{\mathsf T}g_iF_{ij}=g_j,
	\qquad F_{ji}=F_{ij}^{-1}.
	\]
	This structure can be viewed as a finite gluing of locally flat patches: vertices aggregate local coordinate neighborhoods, edges encode finite identifications between neighboring patches, and curvature is detected by the nontrivial holonomy produced when these identifications are composed around closed loops.  More abstractly, it equips a graph with geometrically constrained gauge connection data.
	
	The construction is motivated by a classical fact in differential geometry: for a small loop shrinking to a point, the holonomy of parallel transport is controlled at second order by the curvature of the connection.  Thus, if \(\mathcal C=(i,j,k,i)\) is a selected oriented triangular loop, its holonomy
	\[
	H_{\mathcal C}=F_{ik}F_{kj}F_{ji}
	\]
	is an automorphism of the base fibre.  Whenever the principal logarithm is defined, the normalized quantity
	\[
	\Omega_{\mathcal C}
	=-\frac{1}{s_{\mathcal C}}Log(H_{\mathcal C})
	\]
	gives a curvature observation at a finite loop scale.  Here \(s_{\mathcal C}\) is an effective loop scale compatible with the graph geometry: when embedding or area information is available, it may be induced by a local area; on an abstract graph, it is an explicitly specified normalization structure.  The purpose is therefore not to claim a unique curvature for a purely combinatorial graph, but to discretize the continuous holonomy--curvature mechanism into computable and composable graph-level objects.
	
	The curvature observation \(\Omega_{\mathcal C}\) is \(g_i\)-skew-adjoint, whereas a velocity of a positive-definite metric must be symmetric.  To connect these two types of objects, we introduce a symmetric response matrix \(B_{\mathcal C,i}\) and define a Ricci-type response through the commutator
	\[
	\operatorname{Ric}^{D}_{g}(i)
	=
	\sum_{\mathcal C\in\mathcal C_i}
	\alpha_{\mathcal C}
	[B_{\mathcal C,i},g_i\Omega_{\mathcal C}].
	\]
	This commutator does not redefine curvature.  Rather, it describes a noncommutative coupling between the local curvature channel and the metric-response channel, thereby producing a direction that can drive an update of the positive-definite metric.
	
	In addition to directly aggregating loop curvature at a vertex, we also preserve the incidence structure between oriented faces and edges.  More precisely, loop curvature first induces an antisymmetric edge flux, which is then aggregated at vertices by a covariant graph divergence.  This mechanism corresponds to the continuous relation
	\[
	\operatorname{div}V=\operatorname{tr}_g(\nabla V),
	\]
	where local geometric information is collected through a covariant derivative followed by a trace.  It is related to, but distinct from, the contraction of the curvature tensor that produces the Ricci tensor: the former gives a curvature aggregation that retains face--edge incidence information, while the latter requires additional directional and quadrature structures.  We keep these two mechanisms separate, so that direct aggregation and divergence aggregation can support different graph-geometric interactions.
	
	The resulting structure can be naturally embedded into graph-learning layers.  At each layer, orthogonal factors \(O_{ij}\) generate strictly metric-compatible edge transports by
	\[
	F_{ij}=g_i^{-1/2}O_{ij}g_j^{1/2}.
	\]
	Loop holonomy then produces curvature observations; Ricci-type responses evolve the family \(\{g_i\}\) through an exponential update that preserves positive definiteness; and semantic features may be propagated by message-passing operators conditioned on curvature invariants or other geometric statistics.  In this way, metric evolution and semantic propagation are separated into two coupled but distinct stages: the former constructs and updates local geometry, while the latter uses that geometry to modulate task-dependent information exchange.
	
	The main contributions are as follows.
	\begin{itemize}
		\item We propose a graph geometric structure in which vertices carry local SPD metrics and edges carry reversible metric-compatible discrete parallel transports, together with an orthogonal-factor parametrization and local orthogonal gauge symmetry.
		
		\item Starting from the continuous small-loop holonomy expansion, we construct normalized logarithmic holonomy as a graph curvature observation, and give both direct loop aggregation and a covariant divergence aggregation based on curvature-induced edge fluxes.
		
		\item We use a commutator construction to resolve the mismatch between the skew-adjointness of curvature observations and the symmetry required of metric velocities, obtaining symmetric gauge-equivariant Ricci-type responses and an SPD-preserving discrete update.
		
		\item We provide a structure-preserving learning realization, and use known-geometry calibrations on the unit sphere to test the holonomy--curvature relation, curvature preservation under metric-compatible representations, and the empirical trend of recovering edge transports from local observations.
	\end{itemize}
	
	The rest of the paper is organized as follows.  Section~\ref{sec:structure} defines local metric fibres, edge transports, and oriented loops on graphs.  Section~\ref{sec:curvature} constructs holonomy-based curvature observations and gives both direct and divergence-type aggregations.  Section~\ref{sec:ricci} defines Ricci-type responses and SPD-preserving metric evolution.  Section~\ref{sec:learning} discusses a structure-preserving learning realization.  Section~\ref{sec:checks} gives basic structural checks and special cases.  Section~\ref{sec:synthetic} reports known-geometry calibrations and synthetic consistency experiments.

\section{Graph geometric structures}\label{sec:structure}

\subsection{Local metric fibres and edge transport}

Let $G=(V,E)$ be a finite undirected graph.  We write $j\sim i$ when
$\{i,j\}\in E$, and replace each undirected edge by both directed edges.  At
every node $i$, let $E_i\cong\mathbb R^n$ be a fibre equipped with the inner
product
\[
\langle u,v\rangle_{g_i}=u^{\mathsf T}g_iv,
\qquad g_i\in\SPD_n.
\]
The matrices $g_i$ may be interpreted as node features, as local metric data,
or as learned latent metrics.  The use of an SPD feature space is consistent
with the affine-invariant geometry of positive-definite matrices
\cite{Bhatia2007}.

To motivate a geometric realization, one may regard the graph as
encoding a \emph{piecewise-flat} underlying space: a node $i$ represents a
locally flat patch $U_i$ equipped with a chosen parallel frame, and in that
frame the metric is represented by a constant matrix $g_i\in\SPD_n$.  An edge
records the adjacency of two such patches and carries a finite identification
of their frames.  The corresponding curvature is not asserted to occur inside
an individual flat patch; it is detected by the failure of these
identifications to compose trivially around a loop encircling interfaces or
hinges.

Under this interpretation, $F_{ij}$ is a discrete parallel transport from the
fibre over $j$ to the fibre over $i$.  It is not, in general, the coordinate
transition Jacobian $\varphi_i\varphi_j^{-1}$: the latter satisfies an atlas
cocycle relation on triple overlaps, whereas the former is the connection data
whose loop product records curvature.  A metric-compatible family of such
finite transports, supplemented when appropriate by a discrete torsion-free
condition, is therefore Levi--Civita-inspired in this piecewise-flat
realization.  More generally, the properties of a connection in our framework
are encoded through the parallel transports it induces along graph edges.

\begin{definition}[Metric graph connection]\label{def:connection}
	A \underline{metric graph connection} on $(G,\{g_i\})$ is a family of invertible
	edge maps
	\[
	F_{ij}:E_j\longrightarrow E_i, \qquad j\sim i,
	\]
	such that
	\begin{equation}\label{eq:metric-compatibility}
		F_{ij}^{\mathsf T}g_iF_{ij}=g_j,
		\qquad F_{ji}=F_{ij}^{-1}.
	\end{equation}
\end{definition}

The first condition is the discrete counterpart of metric compatibility: the
transport preserves fibre inner products.  It is a modelling constraint, not a
consequence of an arbitrary choice of edge maps.  

\begin{proposition}[Orthogonal-equivariant]\label{prop:orthogonalequivariant}
	For each directed edge $j\to i$, the metric compatibility condition in
	\eqref{eq:metric-compatibility} holds if and only if
	\begin{equation}\label{eq:transport-parameterization}
		F_{ij}=g_i^{-1/2}O_{ij}g_j^{1/2}
		\quad\text{for some}\quad O_{ij}\in\Orth(n).
	\end{equation}
	Moreover, edge reversal is enforced by choosing $O_{ji}=O_{ij}^{\mathsf T}$.
\end{proposition}

\begin{proof}
	Set $O_{ij}=g_i^{1/2}F_{ij}g_j^{-1/2}$.  Equation
	\eqref{eq:metric-compatibility} gives
	$O_{ij}^{\mathsf T}O_{ij}=I$, proving $O_{ij}\in\Orth(n)$ and
	\eqref{eq:transport-parameterization}.  Conversely, direct substitution of
	\eqref{eq:transport-parameterization} gives
	$F_{ij}^{\mathsf T}g_iF_{ij}=g_j$.  The inverse relation follows from
	$O_{ji}=O_{ij}^{\mathsf T}$. 
\end{proof}

In the piecewise-flat realization above, $F_{ij}$ may be viewed as
a finite parallel transport between adjacent locally flat patches.  On an
abstract graph this interpretation is additional geometric structure rather
than automatic data; Definition~\ref{def:connection} nevertheless remains
meaningful as a metric-compatible discrete gauge connection.

For completeness, we retain an optional discrete analogue of torsion
freeness.  In Cartan's moving-frame notation, the torsion of a connection $A$
is
\[
\Theta(A):=de+\omega_A\wedge e,
\]
where $e$ is the frame field and $\omega_A$ is the connection one-form.  Thus
the torsion-free condition is $\Theta(A)=0$.  Suppose now that each directed
edge $i\to j$ is assigned a displacement vector $v_{ij}\in E_i$.  For a
triangular loop based at $i$, the three terms below all belong to $E_i$.

\begin{definition}
	If the graph assigns each directed edge $i\to j$ a displacement vector
	$v_{ij}\in E_i$, the first structure equation motivates the
	\underline{weak torsion-free condition} for the parallel transports
	$\{F_{ij}\}$ that, for every triangular loop
	\[
	(i,j,k):=\{i\to j\to k\to i\mid i,j,k\in V\},
	\]
	\begin{equation}\label{fij:torsion free}
		v_{ij}+F_{ij}v_{jk}+F_{ik}v_{ki}=0.
	\end{equation}
\end{definition}

For abstract graphs without a spatial embedding, the displacement vectors and
condition \eqref{fij:torsion free} are additional modelling data and should not
be imposed by default.  We therefore retain metric compatibility as the
standing requirement.  Dropping the optional weak torsion-free condition gives
a general metric-compatible connection, while learning the edge factors
$O_{ij}$ may model geometric effects not represented by the abstract graph.
In the piecewise-flat interpretation, imposing both compatibility and the
weak closure condition is the sense in which the construction most closely
resembles a discrete Levi--Civita connection.

\subsection{Loops, scales, and gauge changes}

We use selected triangular faces as minimal graph loops.  Let
\[
\mathscr T:=
\bigl\{\{i,j,k\}\subseteq V:
\{i,j\},\{j,k\},\{k,i\}\in E\bigr\}
\]
be the collection of underlying triangular faces of $G$.  For every
$\tau=\{i,j,k\}\in\mathscr T$, choose exactly one of its two cyclic
orientations.  The resulting oriented face collection is denoted by $\mathcal F\subseteq V^3$. Thus, for each $\tau\in\mathscr T$, exactly one of $(i,j,k)$ and $(i,k,j)$, up to cyclic permutation, belongs to $\mathcal F$.  The ordered triple $\sigma=(i,j,k)\in\mathcal F$ represents the
oriented closed walk $\cC=(i,j,k,i)$.  If $i\in\sigma$, write
$\sigma_i=(i,j,k)$ for the cyclic re-rooting of $\sigma$ whose first vertex is
$i$, and write
\[
\cC_{\sigma,i}:=(i,j,k,i)
\]
for the corresponding rooted closed loop.  Thus $\sigma$ emphasizes the
oriented face, while $\cC_{\sigma,i}$ emphasizes the associated loop and its
base fibre.

For $\sigma_i=(i,j,k)$, define the holonomy of $\sigma$ based at $i$ by
\begin{equation}\label{eq:holonomy}
	H_{\sigma_i}=H_{\cC_{\sigma,i}}:E_i\longrightarrow E_i,
	\qquad
	H_{\sigma_i}:=F_{ik}F_{kj}F_{ji}.
\end{equation}
Metric compatibility implies
$H_{\sigma_i}^{\mathsf T}g_iH_{\sigma_i}=g_i$.  Thus loop holonomy lies in
the $g_i$-orthogonal group.  We use the principal matrix logarithm only under
the standard admissibility condition that $H_{\sigma_i}$ has no eigenvalue on
the closed negative real axis.  This condition excludes the branch ambiguity
at $-1$ and guarantees a real logarithm compatible with the chosen local
branch.

The relation between holonomy and baseline geometry is stressed in the following lemma.
\begin{lemma}
	\label{lem:infinitesimal-holonomy}
	Let $\pi:(E,\nabla)\to M$ be a smooth vector bundle with connection, let
	$p\in M$, and let $X,Y\in T_pM$. Choose a smooth local surface
	$\Phi$ with
	\[
	\Phi(0,0)=p,\qquad
	\partial_s\Phi(0,0)=X,\qquad
	\partial_t\Phi(0,0)=Y.
	\]
	For $\varepsilon>0$, let $\gamma_\varepsilon$ be the positively oriented
	boundary of $\Phi([0,\varepsilon]^2)$, based at $p$, and denote by
	$P_{\gamma_\varepsilon}:E_p\to E_p$ the parallel transport around
	$\gamma_\varepsilon$.  Fix the curvature convention
	\[
	R^\nabla(X,Y):=\nabla_X\nabla_Y-\nabla_Y\nabla_X-\nabla_{[X,Y]}.
	\]
	With the orientation and parallel-transport convention chosen so that the
	displayed sign holds, one has
	\[
	P_{\gamma_\varepsilon}
	=\Id-\varepsilon^2R^\nabla_p(X,Y)+O(\varepsilon^3).
	\]
	Consequently, for sufficiently small $\varepsilon$,
	\[
	\Logm(P_{\gamma_\varepsilon})
	=-\varepsilon^2R^\nabla_p(X,Y)+O(\varepsilon^3).
	\]
\end{lemma}

\begin{proof}
	This is the classical local holonomy formula; see, for example,
	\cite[Chapter~II]{KobayashiNomizu1963}. The logarithmic expansion follows
	from $\Logm(I+A)=A+O(\|A\|^2)$ as $A\to0$.
\end{proof}

Lemma~\ref{lem:infinitesimal-holonomy} is the continuous mechanism
underlying our construction: curvature is recovered from the normalized
logarithmic holonomy of a shrinking loop. Selected graph loops are finite
proxies for such local loops. Accordingly, for a selected loop
$\mathcal C$ with holonomy $H_{\mathcal C}$, we represent the curvature tensor acting on the in--and--out direction pair of loop $\mathcal{C}$ by
\[
\Omega_{\mathcal C}
:=-\frac{1}{s_{\mathcal C}}\Logm(H_{\mathcal C}),
\]
where $s_{\mathcal C}>0$ is an effective curvature scale determined by
the chosen graph geometry. When a local embedding or a controlled cover
is available, $s_{\mathcal C}$ may be chosen from the corresponding
positive (unsigned) area scale; otherwise it is specified through the loop
normalization.

\begin{definition}[Admissible loop curvature]\label{def:loop-curvature}
	For an admissible oriented face $\sigma=(i,j,k)\in\mathcal F$, its
	re-rooting $\sigma_i=(i,j,k)$, and the associated loop
	$\cC_{\sigma,i}=(i,j,k,i)$, set
	\[
	\Xi_{\sigma,i}:=\Logm(H_{\sigma,i}),
	\qquad
	\Xi_{\cC_{\sigma,i}}:=\Logm(H_{\cC_{\sigma,i}})
	=\Xi_{\sigma,i}.
	\]
	Let $s_\sigma=s_{\cC_{\sigma,i}}>0$ be an effective loop scale.  The
	normalized loop-curvature observation based at $i$ is
	\[
	\Omega_{\sigma,i}:=-\frac{1}{s_\sigma}\Xi_{\sigma,i},
	\qquad
	\Omega_{\cC_{\sigma,i}}
	:=-\frac{1}{s_{\cC_{\sigma,i}}}\Xi_{\cC_{\sigma,i}}
	=\Omega_{\sigma,i}.
	\]
\end{definition}

Because $H_{\cC}$ preserves $g_i$, its logarithm satisfies
\begin{equation}\label{eq:gskew}
	\Xi_{\cC}^{\mathsf T}g_i+g_i\Xi_{\cC}=0,
	\qquad
	\Omega_{\cC}^{\mathsf T}g_i+g_i\Omega_{\cC}=0.
\end{equation}
Thus $\Xi_\mathcal{C}$ and $\Omega_\mathcal{C}$ are both
$g_i$-skew-adjoint; equivalently, $g_i\Xi_\mathcal{C}$ and
$g_i\Omega_\mathcal{C}$ are skew-symmetric.

On a graph, $s_{\cC}$ is an explicit modelling choice.  If displacement
vectors or an embedding provide an area $A(\cC)$, one possible effective scale
is
\begin{equation}\label{eq:area-normalization}
	s_{\cC}=(2A(\cC)+\varepsilon_A)\exp(\gamma r_{\cC}^{2}),
	\qquad \gamma,\varepsilon_A>0,
\end{equation}
where $r_{\cC}$ measures loop size.  Then the factor $s_{\cC}^{-1}$ gives a
decaying area-normalized weight.  If no geometric embedding is given, the
normalization should be understood as an effective graph-scale choice rather
than as a literal area reciprocal.

The construction has a natural local orthogonal gauge symmetry: for
$Q_i\in\Orth(n)$ at every node, change the fibre coordinates by
\begin{equation}\label{eq:gauge-action}
	g_i' = Q_ig_iQ_i^{\mathsf T}, \qquad
	F_{ij}'=Q_iF_{ij}Q_j^{\mathsf T}.
\end{equation}
Then $H_{\cC}'=Q_iH_{\cC}Q_i^{\mathsf T}$ and, on the selected logarithm
branch, $\Omega_{\cC}'=Q_i\Omega_{\cC}Q_i^{\mathsf T}$.  This equivariance
will be inherited by the metric responses below.  We restrict to orthogonal
gauge changes because the commutator construction is expressed using matrix
products after lowering an index; general linear coordinate covariance would
require a more carefully typed tensor formulation.

\section{Holonomy-based curvature observations}\label{sec:curvature}

\subsection{From loop holonomy to a curvature signal}

For each vertex $i$, let
\[
\mathcal F_i:=\{\sigma\in\mathcal F: i\in\sigma\},
\qquad
\cC_i:=\{\cC_{\sigma,i}:\sigma\in\mathcal F_i\}.
\]
The direct holonomy aggregation is
\begin{equation}\label{eq:direct-aggregate}
	K_i:=\sum_{\sigma\in\mathcal F_i}
	\alpha_{\sigma,i}\Omega_{\sigma,i},
	\qquad \alpha_{\sigma,i}=:\alpha_{\mathcal C_{\sigma,i}}\geq0.
\end{equation}
Each selected underlying triangular face occurs exactly once in
$\mathcal F$.  Its orientation is encoded in the oriented loop defining
$\Omega_{\sigma,i}$, while $\alpha_{\sigma,i}$ is an unsigned scale or
confidence weight.  Since each summand is $g_i$-skew-adjoint, $K_i$ is also
$g_i$-skew-adjoint. In a smooth setting, the curvature tensor
is a two-form with values in endomorphisms, and the Ricci tensor arises only
after a specified contraction.  Equation~\eqref{eq:direct-aggregate} does not
recover this canonical contraction on an arbitrary graph: it is an aggregated
loop-curvature signal.  The distinction is central to the interpretation of
our method.

There are two ways to give $\alpha_{\cC}$ a geometric meaning.  With an
embedded graph and displacement vectors, a loop area and a local frame can
motivate the scale in \eqref{eq:area-normalization}.  Recall that under a fixed global trivialization, the node metrics admit a common matrix representation, and one may use the affine-invariant distance
\begin{equation}\label{eq:spd-distance}
	\dist_{\mathrm{SPD}}(g_i,g_j)
	=\bigl\|\log(g_i^{-1/2}g_jg_i^{-1/2})\bigr\|_F
\end{equation}
to form an effective loop size, for example
\[
r_{ijk}^{2}=\dist_{\mathrm{SPD}}(g_i,g_j)^2+
\dist_{\mathrm{SPD}}(g_j,g_k)^2+
\dist_{\mathrm{SPD}}(g_k,g_i)^2.
\]
This option is useful when the node matrices share a reference basis.  Under
independent local gauges, however, the direct comparison in
\eqref{eq:spd-distance} is not an invariant quantity.  It must therefore not
be presented as an intrinsic gauge-invariant edge scalar without an additional
registration assumption.

\begin{proposition}[Frame-based Ricci contraction calibration]
	\label{prop:ricci-contraction-calibration}
	Suppose that the graph data near a node $i$ are obtained from a local
	sampling of a Riemannian manifold $(M,g)$, and identify $E_i$ with
	$T_{p_i}M$.  Let $\mathcal N_i^{\mathrm{tr}}\subseteq\mathcal N(i)$ denote the set of
	neighbours selected for the discrete trace approximation, and $\{(v_{ik},\rho_{ik})\}_{k\in\mathcal N_i^{\mathrm{tr}}}$
	be a weighted family in $T_{p_i}M$ satisfying the resolution of identity
	\begin{equation}\label{eq:frame-resolution}
		\sum_{k\in\mathcal N_i^{\mathrm{tr}}}
		\rho_{ik}\,v_{ik}\otimes v_{ik}^{\flat}
		=\Id_{T_{p_i}M},
		\qquad
		v_{ik}^{\flat}(Z):=g_i(v_{ik},Z).
	\end{equation}
	For each sampled direction $v_{ij}$ and each
	$k\in\mathcal N_i^{\mathrm{tr}}$, let $\mathcal C_{ikj}$ be a selected
	oriented local loop that admits a smooth realization as the boundary of a
	shrinking surface to which Lemma~\ref{lem:infinitesimal-holonomy} applies.
	Assume that its scale is calibrated by
	$s_{\mathcal C_{ikj}}=\varepsilon_{ikj}^{2}
	+O(\varepsilon_{ikj}^{3})$.  Then the lemma and the normalization in
	Definition~\ref{def:loop-curvature} give
	\begin{equation}\label{eq:loop-curvature-calibration}
		\Omega_{\mathcal C_{ikj}}
		=R_{p_i}(v_{ik},v_{ij})+E_{ikj},
		\qquad
		E_{ikj}=O(\varepsilon_{ikj}).
	\end{equation}
	where $R$ is the curvature of the underlying connection and $E_{ikj}$ is
	the loop-curvature approximation error.  Define
	\begin{equation}\label{eq:discrete-ricci-contraction}
		\widehat{\operatorname{Ric}}_i(v_{ij},Y)
		:=
		\sum_{k\in\mathcal N_i^{\mathrm{tr}}}
		\rho_{ik}\,
		g_i\bigl(\Omega_{\mathcal C_{ikj}}Y,v_{ik}\bigr).
	\end{equation}
	Then
	\begin{equation}\label{eq:ricci-contraction-error}
		\widehat{\operatorname{Ric}}_i(v_{ij},Y)
		=
		\operatorname{Ric}_{p_i}(v_{ij},Y)
		+
		\sum_{k\in\mathcal N_i^{\mathrm{tr}}}
		\rho_{ik}\,
		g_i(E_{ikj}Y,v_{ik}).
	\end{equation}
\end{proposition}

\begin{proof}
	Let $A_Y:Z\mapsto R_{p_i}(Z,v_{ij})Y$.  By
	\eqref{eq:frame-resolution},
	\[
	\operatorname{tr}(A_Y)
	=
	\sum_{k\in\mathcal N_i^{\mathrm{tr}}}
	\rho_{ik}\,g_i\bigl(A_Yv_{ik},v_{ik}\bigr).
	\]
	Hence
	\[
	\operatorname{Ric}_{p_i}(v_{ij},Y)
	=
	\sum_{k\in\mathcal N_i^{\mathrm{tr}}}
	\rho_{ik}\,
	g_i\bigl(R_{p_i}(v_{ik},v_{ij})Y,v_{ik}\bigr).
	\]
	Substituting \eqref{eq:loop-curvature-calibration} yields
	\eqref{eq:ricci-contraction-error}.
\end{proof}

\begin{remark}
	The weights $\rho_{ik}$ in \eqref{eq:frame-resolution} are trace-quadrature
	weights and should not be conflated with the edge interaction weights
	$w_{ik}$ or the loop aggregation weights $\alpha_{\mathcal C}$.  The
	proposition gives a direction-resolved discrete proxy for the Ricci
	contraction.  By contrast, the operator $K_i$ in
	\eqref{eq:direct-aggregate} is a direction-compressed loop-curvature signal;
	the commutator construction in Definition~\ref{def:ricci-type} converts such
	a signal into a symmetric metric response.
\end{remark}

\subsection{An incidence-based covariant divergence}\label{sec:covariant-divergence}

The direct aggregate \eqref{eq:direct-aggregate} retains loop curvature at a
vertex but discards how individual loops are incident on the neighbouring
edges.  We now organize the same local observations through a covariant
edge-flux construction.  The geometric motivation is related to, but distinct
from, the Ricci contraction discussed above.  On a Riemannian manifold,
divergence is itself a trace operation:
\[
\operatorname{div}V=\operatorname{tr}_g(\nabla V)
\]
for a vector field $V$, whereas the Ricci tensor is the trace of the different
endomorphism $Z\mapsto R(Z,X)Y$.  Thus both operations contract local
geometric information, but they contract different objects.  The construction
below discretizes the former mechanism: oriented faces create
curvature-induced edge fluxes, which are then collected at a vertex by a
covariant divergence.  It is not presented as a second definition of the
Ricci contraction.

We now use the oriented face collection $\mathcal F$ fixed above; each
selected underlying triangular face occurs exactly once.  If
$\sigma\in\mathcal F$ is incident on a vertex $i$, let
$\Omega_{\sigma,i}\in\operatorname{End}(E_i)$ denote its normalized
logarithmic holonomy based at $i$, obtained by cyclically re-rooting the
oriented loop at $i$.  For example, if $\sigma_i=(i,j,k)$, so that
$\cC_{\sigma,i}=(i,j,k,i)$, then
\begin{equation}\label{eq:rooted-holonomy}
	\Omega_{\sigma,i}=-\frac{1}{s_\sigma}
	\Logm(F_{ik}F_{kj}F_{ji}),
	\qquad
	\Omega_{\sigma,j}=F_{ji}\Omega_{\sigma,i}F_{ij}.
\end{equation}
The second identity follows from $F_{ji}=F_{ij}^{-1}$ and the similarity
equivariance of the selected logarithm branch.  In particular, each
$\Omega_{\sigma,i}$ is $g_i$-skew-adjoint.

For an oriented edge $i\to j$ contained in $\sigma$, let
$\epsilon_{ij}(\sigma)=1$ when the boundary orientation of $\sigma$ traverses
$i\to j$, and let $\epsilon_{ij}(\sigma)=-1$ otherwise.  Thus
$\epsilon_{ji}(\sigma)=-\epsilon_{ij}(\sigma)$.  Write
$\mathcal F_{ij}$ for the faces containing the underlying edge $\{i,j\}$, and
define the curvature edge cochain based at $i$ by
\begin{equation}\label{eq:curvature-edge-cochain}
	A_{ij}:=\sum_{\sigma\in\mathcal F_{ij}}
	\epsilon_{ij}(\sigma)\Omega_{\sigma,i}
	\in\operatorname{End}(E_i).
\end{equation}
Although the subscript $ij$ records the graph direction $i\to j$, the
coefficient $A_{ij}$ acts on the source fibre $E_i$.  Acting on a vector
$s_i\in E_i$, the quantity $A_{ij}s_i$ is the corresponding curvature-induced
outgoing edge displacement.  The edge cochain $A$, rather than an evaluation
$A_{ij}s_i$ at an arbitrary section, is the object that has the required
covariant antisymmetry.

\begin{proposition}[Covariant flux identities]\label{prop:covariant-flux}
	The curvature edge cochain \eqref{eq:curvature-edge-cochain} satisfies
	\begin{equation}\label{eq:covariant-antisymmetry}
		A_{ji}=-F_{ji}A_{ij}F_{ij}.
	\end{equation}
	Let $w_{ij}=w_{ji}\geq0$ be edge conductances and let $\mu_i>0$ be vertex
	masses.  The covariant graph divergence of $A$ is the endomorphism
	\begin{equation}\label{eq:covariant-divergence}
		K_i^{\mathrm{div}}:=(\operatorname{div}_{w,\mu}^{F}A)_i
		:=\frac{1}{\mu_i}\sum_{j\sim i}w_{ij}A_{ij}
		\in\operatorname{End}(E_i).
	\end{equation}
	It is $g_i$-skew-adjoint and is equivariant under the local orthogonal gauge
	action \eqref{eq:gauge-action}.  More precisely,
	\begin{equation}\label{eq:divergence-gauge}
		(K_i^{\mathrm{div}})'=Q_iK_i^{\mathrm{div}}Q_i^{\mathsf T}.
	\end{equation}
	
	Moreover, for an endomorphism-valued vertex cochain
	$B=\{B_i\in\operatorname{End}(E_i)\}_{i\in V}$, define its covariant edge
	increment by
	\begin{equation}\label{eq:covariant-increment}
		(\mathrm d^F B)_{ij}:=F_{ij}B_jF_{ji}-B_i\in\operatorname{End}(E_i).
	\end{equation}
	Endow $\operatorname{End}(E_i)$ with the $g_i$-Hilbert--Schmidt pairing
	\begin{equation}\label{eq:g-hs-pairing}
		\langle S,T\rangle_{g_i}:=
		\operatorname{tr}\bigl(S^{\dagger_{g_i}}T\bigr),
		\qquad
		S^{\dagger_{g_i}}:=g_i^{-1}S^{\mathsf T}g_i.
	\end{equation}
	Then, for any choice of one orientation of each undirected edge,
	\begin{equation}\label{eq:covariant-green}
		\sum_{i\in V}\mu_i\langle B_i,(\operatorname{div}_{w,\mu}^{F}A)_i\rangle_{g_i}
		=-\sum_{\{i,j\}\in E}w_{ij}
		\langle(\mathrm d^F B)_{ij},A_{ij}\rangle_{g_i}.
	\end{equation}
\end{proposition}

\begin{proof}
	The rooted holonomies of a fixed oriented face are related by transport as in
	\eqref{eq:rooted-holonomy}.  Together with
	$\epsilon_{ji}(\sigma)=-\epsilon_{ij}(\sigma)$, this gives
	\[
	A_{ji}=\sum_{\sigma\in\mathcal F_{ij}}
	\epsilon_{ji}(\sigma)\Omega_{\sigma,j}
	=-F_{ji}\left(\sum_{\sigma\in\mathcal F_{ij}}
	\epsilon_{ij}(\sigma)\Omega_{\sigma,i}\right)F_{ij},
	\]
	which proves \eqref{eq:covariant-antisymmetry}.  Since each summand in
	\eqref{eq:curvature-edge-cochain} is $g_i$-skew-adjoint, the same is true of
	$A_{ij}$ and of the weighted sum \eqref{eq:covariant-divergence}.  Gauge
	equivariance follows because every factor in
	\eqref{eq:curvature-edge-cochain} based at $i$ transforms by conjugation with
	$Q_i$; the scalar weights $w_{ij}$ and $\mu_i$ are gauge invariant.
	
	Metric compatibility makes $F_{ji}:(E_i,g_i)\to(E_j,g_j)$ an isometry.
	Consequently, conjugation by $F_{ji}$ preserves the pairings
	\eqref{eq:g-hs-pairing}.  Using \eqref{eq:covariant-antisymmetry} and grouping
	the left-hand side of \eqref{eq:covariant-green} over undirected edges yields
	\[
	\begin{aligned}
		\sum_{i\in V}\mu_i\langle B_i,(\operatorname{div}_{w,\mu}^{F}A)_i\rangle_{g_i}
		&=\sum_{\{i,j\}\in E}w_{ij}
		\left(\langle B_i,A_{ij}\rangle_{g_i}
		+\langle B_j,A_{ji}\rangle_{g_j}\right) \\
		&=\sum_{\{i,j\}\in E}w_{ij}
		\left\langle B_i-F_{ij}B_jF_{ji},A_{ij}\right\rangle_{g_i},
	\end{aligned}
	\]
	which is \eqref{eq:covariant-green} by \eqref{eq:covariant-increment}.
\end{proof}

Proposition~\ref{prop:covariant-flux} is the discrete analogue of the
divergence mechanism used here: face curvature is first assigned to its
oriented boundary edges, and the signed outgoing fluxes are then summed at a
vertex.  Equation~\eqref{eq:covariant-green} identifies this operation as the
negative adjoint of a covariant edge increment, exactly as divergence is the
negative adjoint of a gradient after an integration-by-parts convention is
fixed.  The proposition does not identify $K_i^{\mathrm{div}}$ with a Ricci
tensor; it supplies a geometrically structured alternative to the direct local
aggregation \eqref{eq:direct-aggregate}.

\begin{remark}[Choice of scales, conductances, and masses]\label{rem:weights}
	The inverse loop scales $s_\sigma^{-1}$ determine the scale of individual
	curvature observations, whereas $w_{ij}$ determines how edge fluxes are
	aggregated at a
	vertex.  These roles should not be conflated.  The formal construction requires
	only symmetric nonnegative conductances and positive masses; one may take
	$\mu_i=1$ in the absence of a vertex-volume model.  When the node metrics are
	represented in a shared reference frame, a possible conductance prior is
	\[
	w_{ij}=\exp\bigl(-\beta\dist_{\mathrm{SPD}}(g_i,g_j)^2\bigr),
	\qquad \beta>0.
	\]
	For a locally gauge-equivariant model without such a shared representation,
	$w_{ij}$ should instead be supplied by gauge-invariant edge data or by a
	separately designed equivariant module.  Keeping these distinctions explicit
	prevents a coordinate-dependent prior from being mistaken for intrinsic graph
	geometry.
\end{remark}

\section{Ricci-type metric responses and evolution}\label{sec:ricci}

\subsection{Symmetric rectification of loop curvature}

The lower-index matrix $g_iK_i$ is antisymmetric by \eqref{eq:gskew}.  Hence
neither $g_iK_i$ nor its direct sum can be a velocity of a symmetric metric.
This is the obstruction that separates a loop-curvature observation from a
Ricci-type tensor.  The following elementary construction resolves the
symmetry mismatch.

\begin{definition}[Commutator Ricci-type response]\label{def:ricci-type}
	Let $B_{\cC,i}\in\Sym_n$ be a symmetric response matrix assigned to a loop
	$\cC\in\cC_i$.  Define
	\begin{equation}\label{eq:ricci-type}
		\RicD_g(i):=
		\sum_{\cC\in\cC_i}\alpha_{\cC}
		\bigl(B_{\cC,i}g_i\Omega_{\cC}
		-g_i\Omega_{\cC}B_{\cC,i}\bigr).
	\end{equation}
	If a single response matrix $B_i\in\Sym_n$ is used at node $i$, this becomes
	\begin{equation}\label{eq:ricci-type-node}
		\RicD_g(i)=[B_i,g_iK_i].
	\end{equation}
	The divergence-type version is
	\begin{equation}\label{eq:ricci-divergence}
		\RicDiv_g(i):=[B_i,g_iK_i^{\mathrm{div}}],
	\end{equation}
	where $K_i^{\mathrm{div}}$ is the covariant divergence of the curvature edge
	cochain defined in \eqref{eq:covariant-divergence}.
\end{definition}

Proposition~\ref{prop:ricci-contraction-calibration} explains the
geometric role of local loop aggregation. In the smooth calibration regime,
a normalized logarithmic holonomy records curvature sampled in two local
directions, while aggregation over the local loop family combines the
available directional observations. When the discrete directions satisfy a
frame-resolution condition, this mechanism approximates the Ricci
contraction. On a general abstract graph, however, such a condition is not
canonical; hence $K_i$ is regarded as a local curvature aggregate rather
than a discrete Ricci tensor.

\begin{proposition}[Symmetry and gauge equivariance]\label{prop:symmetry-gauge}
	For every $i$, the matrices $\RicD_g(i)$ and $\RicDiv_g(i)$ are symmetric.
	If the data are transformed by \eqref{eq:gauge-action} and the response
	matrices obey $B_{\cC,i}'=Q_iB_{\cC,i}Q_i^{\mathsf T}$, then
	\begin{equation}\label{eq:ricci-gauge}
		\RicD_{g'}(i)=Q_i\RicD_g(i)Q_i^{\mathsf T},
		\qquad
		\RicDiv_{g'}(i)=Q_i\RicDiv_g(i)Q_i^{\mathsf T}.
	\end{equation}
\end{proposition}

\begin{proof}
	For symmetric $B$ and $g_i$-skew-adjoint $\Omega$, the matrix $g_i\Omega$ is
	antisymmetric.  Consequently
	\[
	[B,g_i\Omega]^{\mathsf T}
	=(B g_i\Omega-g_i\Omega B)^{\mathsf T}
	=B g_i\Omega-g_i\Omega B.
	\]
	Summation proves symmetry.  Under \eqref{eq:gauge-action}, both $B$ and
	$g_i\Omega$ transform by $Q_i(\cdot)Q_i^{\mathsf T}$; the commutator and each
	weighted sum therefore transform as in \eqref{eq:ricci-gauge}. 
\end{proof}

The most basic nonlearnable choice is $B_{\cC,i}=g_i$.  It gives
\begin{equation}\label{eq:basic-response}
	\RicD_g(i)=\sum_{\cC\in\cC_i}\alpha_{\cC}
	[g_i,g_i\Omega_{\cC}].
\end{equation}
More flexibly, $B_{\cC,i}$ can be a learned symmetric matrix.  This parameter
does not create curvature: curvature remains encoded in loop holonomy.  Rather,
$B_{\cC,i}$ specifies how the curvature observation couples to the metric
degrees of freedom.  The formula can vanish for commuting $B_{\cC,i}$ and
$g_i\Omega_{\cC}$, which is an expected feature of a commutator-based response,
not an error to be hidden.

\subsection{SPD-preserving metric evolution}

Since the tangent space at an SPD matrix consists of symmetric matrices,
Proposition~\ref{prop:symmetry-gauge} makes the response eligible to drive a
metric update.  A Euclidean update $g_i^+=g_i-2\eta\RicD_g(i)$ may nevertheless
leave the SPD cone.  We therefore use the exponential retraction
\begin{equation}\label{eq:spd-update}
	g_i^{+}=g_i^{1/2}
	\exp\!\left(-2\eta\,g_i^{-1/2}\RicD_g(i)g_i^{-1/2}\right)g_i^{1/2},
	\qquad \eta>0.
\end{equation}
The matrix in the exponential is symmetric, so \eqref{eq:spd-update} is SPD
for every step size.  Replacing $\RicD_g$ by $\RicDiv_g$ gives the
divergence-type evolution.  This is a discrete Ricci-\emph{type} flow: it
shares the form of a curvature-driven metric evolution with
the smooth Ricci flow equation
\(\partial_t g_t=-2\operatorname{Ric}(g_t)\), but its generator depends on the
selected graph loops, their normalizations, and the response matrices.

\begin{algorithm}[H]
	\caption{Holonomy-based Ricci-type metric update at one layer}
	\label{alg:update}
	\begin{algorithmic}[1]
		\REQUIRE Graph $G$, metrics $\{g_i\}$, edge transports $\{F_{ij}\}$,
		oriented face set $\mathcal F$ (one orientation per selected underlying triangle) with loop collection $\{\cC_i\}$, symmetric response matrices $\{B_{\cC,i}\}$,
		and step size $\eta$.
		\ENSURE Updated SPD metrics $\{g_i^+\}$.
		\FOR{each oriented triangle $\cC=(i,j,k,i)$}
		\STATE Compute $H_{\cC}=F_{ik}F_{kj}F_{ji}$ and
		$\Omega_{\cC}=-s_{\cC}^{-1}\Logm(H_{\cC})$.
		\ENDFOR
		\FOR{each node $i$}
		\STATE Form $S_i\gets\sum_{\cC\in\cC_i}\alpha_{\cC}
		[B_{\cC,i},g_i\Omega_{\cC}]$.
		\STATE Set $g_i^+\gets g_i^{1/2}\exp(-2\eta
		g_i^{-1/2}S_ig_i^{-1/2})g_i^{1/2}$.
		\ENDFOR
		\RETURN $\{g_i^+\}$.
	\end{algorithmic}
\end{algorithm}

\section{A geometry-aware learning realization}\label{sec:learning}

The preceding sections define a family of geometric layers rather than a fully
specified neural architecture.  This section records one conservative way to
make the free choices learnable while retaining the structural constraints.
Let $h_i^{(\ell)}$ denote semantic node features and
$g_i^{(\ell)}\in\SPD_n$ metric features at layer $\ell$.  A model may learn
orthogonal factors $O_{ij}^{(\ell)}$ and set
\begin{equation}\label{eq:learn-transport}
	F_{ij}^{(\ell)}=(g_i^{(\ell)})^{-1/2}
	O_{ij}^{(\ell)}(g_j^{(\ell)})^{1/2},
	\qquad O_{ij}^{(\ell)}\in\Orth(n),
	\qquad O_{ji}^{(\ell)}=(O_{ij}^{(\ell)})^{\mathsf T},
\end{equation}
so compatibility is exact at every layer.  In practice, an orthogonal factor
can be produced by an exponential of a skew-symmetric matrix or by an
orthogonalization parameterization.

Likewise, a response matrix can be parameterized as
\begin{equation}\label{eq:learn-response}
	B_{\cC,i}^{(\ell)}=L_{\cC,i}^{(\ell)}
	(L_{\cC,i}^{(\ell)})^{\mathsf T}+\delta I,
	\qquad \delta>0,
\end{equation}
which guarantees both symmetry and positive definiteness.  The
curvature response is then inserted into \eqref{eq:spd-update}.  The semantic
features can be updated by an ordinary equivariant or invariant message-passing
operator that is conditioned on scalar loop statistics, for example norms of
$\Omega_{\cC}^{(\ell)}$ or gauge-invariant contractions constructed from them.

If a common registration of node metric coordinates is part of the data model,
one possible symmetric conductance prior and its associated row-normalized
attention are
\begin{equation}\label{eq:attention}
	\widetilde w_{ij}^{(\ell)}=
	\exp\!\left[-\beta_{\ell}\,
	\dist_{\mathrm{SPD}}\bigl(g_i^{(\ell)},g_j^{(\ell)}\bigr)^2\right],
	\qquad
	\pi_{ij}^{(\ell)}=
	\frac{\widetilde w_{ij}^{(\ell)}}
	{\sum_{k\sim i}\widetilde w_{ik}^{(\ell)}}.
\end{equation}
Here $\widetilde w_{ij}^{(\ell)}=\widetilde w_{ji}^{(\ell)}$ may be used as
the conductance $w_{ij}$ in the divergence construction, whereas
$\pi_{ij}^{(\ell)}$ is a directional attention weight for semantic
aggregation.  For unregistered fibres, Equation~\eqref{eq:attention} should not be used as a
gauge-invariant rule.  This is not merely a technical qualification: the
separation between a common-coordinate metric prior and local-frame gauge
symmetry determines what can legitimately be claimed about the model.

The full layer can therefore be read in two stages.  First, edge transports
produce loop holonomies and a symmetric metric response, which updates
$g_i^{(\ell)}$ within $\SPD_n$.  Second, node semantics $h_i^{(\ell)}$ are
propagated using a task-specific operator conditioned on the resulting
geometry.  This division respects the different roles of metric evolution and
semantic aggregation.  It also keeps a learning system from conflating an
arbitrary edge transformation with a metric-compatible parallel transport.

\section{Structural checks and special cases}\label{sec:checks}

This section records several elementary checks that determine how the proposed
objects should behave.  They are not convergence results, but they distinguish
the holonomy construction from an arbitrary matrix-valued message-passing rule.

\paragraph{\textbf{Flat-loop consistency}}\label{prop:flatness}
Suppose that $H_{\cC}=I$ for every selected loop based at $i$.  The principal logarithm of $I$ is zero.  Therefore every loop curvature term
$\Omega_{\cC}$ vanishes, and so do the sums in
\eqref{eq:direct-aggregate} and \eqref{eq:covariant-divergence}.  The
commutator definitions \eqref{eq:ricci-type} and \eqref{eq:ricci-divergence}
then vanish.  In \eqref{eq:spd-update}, the exponential has zero exponent. Thus
\[
K_i=0,\qquad K_i^{\mathrm{div}}=0,\qquad
\RicD_g(i)=0,\qquad \RicDiv_g(i)=0.
\]
Consequently, both updates defined by \eqref{eq:spd-update} leave $g_i$
unchanged.
\paragraph{\textbf{Existence of connection data}}
For arbitrary SPD metrics, metric-compatible transports always exist: choosing
$O_{ij}=I$ in \eqref{eq:transport-parameterization} supplies one.  This choice
is globally flat, however, because the transports telescope around every loop.
Nontrivial curvature is therefore carried by the incompatible edgewise
orthogonal factors $O_{ij}$, not by the node metrics alone.  This observation
has a useful modelling consequence.  An architecture that learns only
$\{g_i\}$ while fixes $O_{ij}=I$ cannot produce nonzero holonomy curvature;
it needs either learned connection degrees of freedom or a transport rule
derived from external geometric data.

At the opposite extreme, a graph with no selected cycles has
$\cC_i=\varnothing$ at every node.  Its direct and divergence-type curvature
responses are zero by definition.  This is mathematically consistent but
restrictive for tree-like graphs.  There are at least three principled
extensions: use a supplied two-dimensional cell complex; construct a local
clique or Vietoris--Rips completion whose added faces are part of the model;
or use longer fundamental cycles with an explicitly chosen scale
normalization.  None of these extensions is silently assumed in the present
triangular construction.

\paragraph{\textbf{Role of the response matrix}}
The response matrix $B$ must be distinguished from the transport connection.
The transport determines $H_{\cC}$ and hence the curvature observation, while
$B$ selects a symmetric channel through which that observation changes the
metric.  Its nonuniqueness is unavoidable in the current graph setting because
a graph alone does not provide a canonical analogue of the smooth tensor
contraction defining the Ricci tensor.  Two illustrative cases are worth
separating:
\begin{itemize}
	\item \emph{Fixed response.}  Taking $B_{\cC,i}=g_i$ yields the completely
	specified symmetric response in \eqref{eq:basic-response}.  This is a useful
	baseline when one wants the fewest learned geometric parameters.
	\item \emph{Learned response.}  Taking $B_{\cC,i}$ from
	\eqref{eq:learn-response} permits loop- and node-dependent coupling while
	preserving symmetry and gauge equivariance.  It enlarges the model class and
	should be evaluated against the fixed baseline rather than treated as a
	theorem about a unique curvature contraction.
\end{itemize}
In either case, the commutator makes clear that the metric response detects
noncommutativity between the curvature channel $g_i\Omega_{\cC}$ and the
selected symmetric channel $B_{\cC,i}$.  It is this spectral mismatch that
drives the proposed evolution.

\section{Geometric calibration and synthetic consistency checks}
\label{sec:synthetic}

This section tests the finite-dimensional mechanism of the construction on a
geometry with known curvature.  The tests have deliberately separated roles.
The first calibrates the passage from loop holonomy to curvature on the unit
sphere; the second verifies that the full
\[
	g_i,\ O_{ij}\longmapsto
	F_{ij}=g_i^{-1/2}O_{ij}g_j^{1/2}
	\longmapsto H_{\cC}\longmapsto\Omega_{\cC}
\]
representation preserves that curvature after comparison in a common
orthonormal frame; and the third gives a synthetic, non-oracle realization of
the learnable factors in \eqref{eq:learn-transport}.  No task labels are used
in these checks.  Their purpose is to make the geometric mechanism
falsifiable, rather than to substitute for a continuum theorem or an
application-level performance evaluation.

Throughout this section, the graph is an oriented triangular icosphere
approximating the unit sphere \(\mathbb S^2\).  The fibre at \(p_i\) is
\(T_{p_i}\mathbb S^2\), represented in a local orthonormal frame.  The
controlled geometry supplies the unsigned spherical area
\(A_{\cC}>0\) of every face, and we make the particular choice
\(s_{\cC}=A_{\cC}\).  This is a known-geometry specialization of the effective
scale in Definition~\ref{def:loop-curvature}; it is not a prescription for
the task-dependent normalization of a general graph.

\subsection{Holonomy--curvature calibration on the sphere}
\label{sec:sphere-calibration}

We first take \(g_i=I\) and let \(O_{ij}\) be the exact Levi--Civita parallel
transport along the shortest great-circle arc from \(p_j\) to \(p_i\).  Thus
\(F_{ij}=O_{ij}\).  For an outward-oriented spherical triangle, let
\(\theta_{\cC}\) denote the signed principal rotation angle of
\(H_{\cC}\), with the sign chosen consistently with the face orientation.  On
the unit sphere, the exact relation is
\[
	\theta_{\cC}=A_{\cC},
	\qquad
	\widehat K_{\cC}:=\frac{\theta_{\cC}}{A_{\cC}}=1.
\]
This scalar presentation is compatible with the matrix convention
\(\Omega_{\cC}=-A_{\cC}^{-1}\Logm(H_{\cC})\): the latter retains the oriented
endomorphism, whereas \(\widehat K_{\cC}\) records its scalar
holonomy--area calibration.

Figure~\ref{fig:sphere-holonomy} reports the calculation at four mesh levels,
from \(20\) to \(1280\) faces.  The largest observed
\(\lvert\theta_{\cC}-A_{\cC}\rvert\) is below \(4\times10^{-16}\), and the
estimated scalar curvature is \(1\) to numerical precision.  The deliberately
incorrect flat control \(F_{ij}=I\) has identically zero loop holonomy, so it
cannot recover the curvature of the sphere.  The refinement panel is included
only as a numerical stability check of the exact construction: since the
transports and areas are analytically supplied, it is not evidence for a
general convergence rate under mesh refinement.

\begin{figure}[H]
	\centering
	\includegraphics[width=\textwidth]{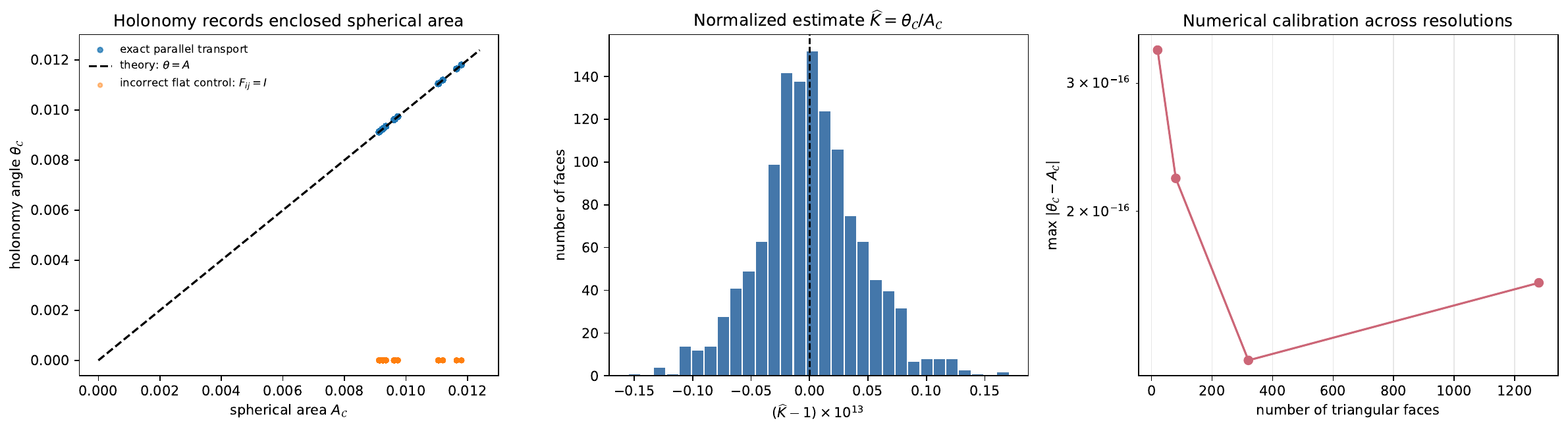}
	\caption{Holonomy--curvature calibration on the unit sphere.  Exact
	Levi--Civita edge transport gives the signed relation
	\(\theta_{\cC}=A_{\cC}\) for every oriented spherical face.  With the
	controlled scale \(s_{\cC}=A_{\cC}\), the scalar estimate
	\(\widehat K_{\cC}=\theta_{\cC}/A_{\cC}\) equals the unit-sphere curvature
	to numerical precision.  The flat control \(F_{ij}=I\) produces zero
	holonomy.}
	\label{fig:sphere-holonomy}
\end{figure}

\subsection{Metric representation calibration}
\label{sec:metric-representation-calibration}

The preceding test isolates the holonomy--curvature relation but uses the
trivial fibre metric.  We next use a smooth, spatially varying SPD field
\(\{g_i\}\) with condition numbers between \(1.003\) and \(2.600\) on the
finest mesh, while retaining the same exact sphere transports
\(\{O_{ij}\}\).  We then form \(F_{ij}\) by
\eqref{eq:transport-parameterization}.  This is the full metric-compatible
representation used in the paper, but it is not asserted to be the
Levi--Civita connection of the auxiliary field \(\{g_i\}\).

For a loop based at \(i\), the metric factors telescope exactly:
\[
	H_{\cC}^{F}
	=g_i^{-1/2}H_{\cC}^{O}g_i^{1/2},
	\qquad
	\widetilde\Omega_{\cC}^{F}
	:=g_i^{1/2}\Omega_{\cC}^{F}g_i^{-1/2}
	=\Omega_{\cC}^{O}.
\]
The first equality makes clear why a raw matrix comparison is not geometric:
\(\Omega_{\cC}^{F}\) and \(\Omega_{\cC}^{O}\) are written in different metric
representations.  On the finest mesh, the unwhitened difference
\(\|\Omega_{\cC}^{F}-\Omega_{\cC}^{O}\|_F\) ranges from
\(2.09\times10^{-3}\) to \(7.21\times10^{-1}\), whereas the whitened
difference
\(\|\widetilde\Omega_{\cC}^{F}-\Omega_{\cC}^{O}\|_F\) is at most
\(3.5\times10^{-13}\); see Figure~\ref{fig:metric-representation}.  The
compatibility, telescoping, and \(g_i\)-skew residuals are also at numerical
precision.  Thus the nontrivial local metrics do not introduce an additional
holonomy distortion when \(O_{ij}\) and \(s_{\cC}\) are fixed.

This conclusion is intentionally limited.  It does not say that \(g_i\) is
irrelevant in the full model: the local metric may still enter the loop scales,
weights, response matrices, and SPD evolution.  It only isolates the
representation identity that underlies the construction
\eqref{eq:transport-parameterization}.

\begin{figure}[H]
	\centering
	\includegraphics[width=0.96\textwidth]{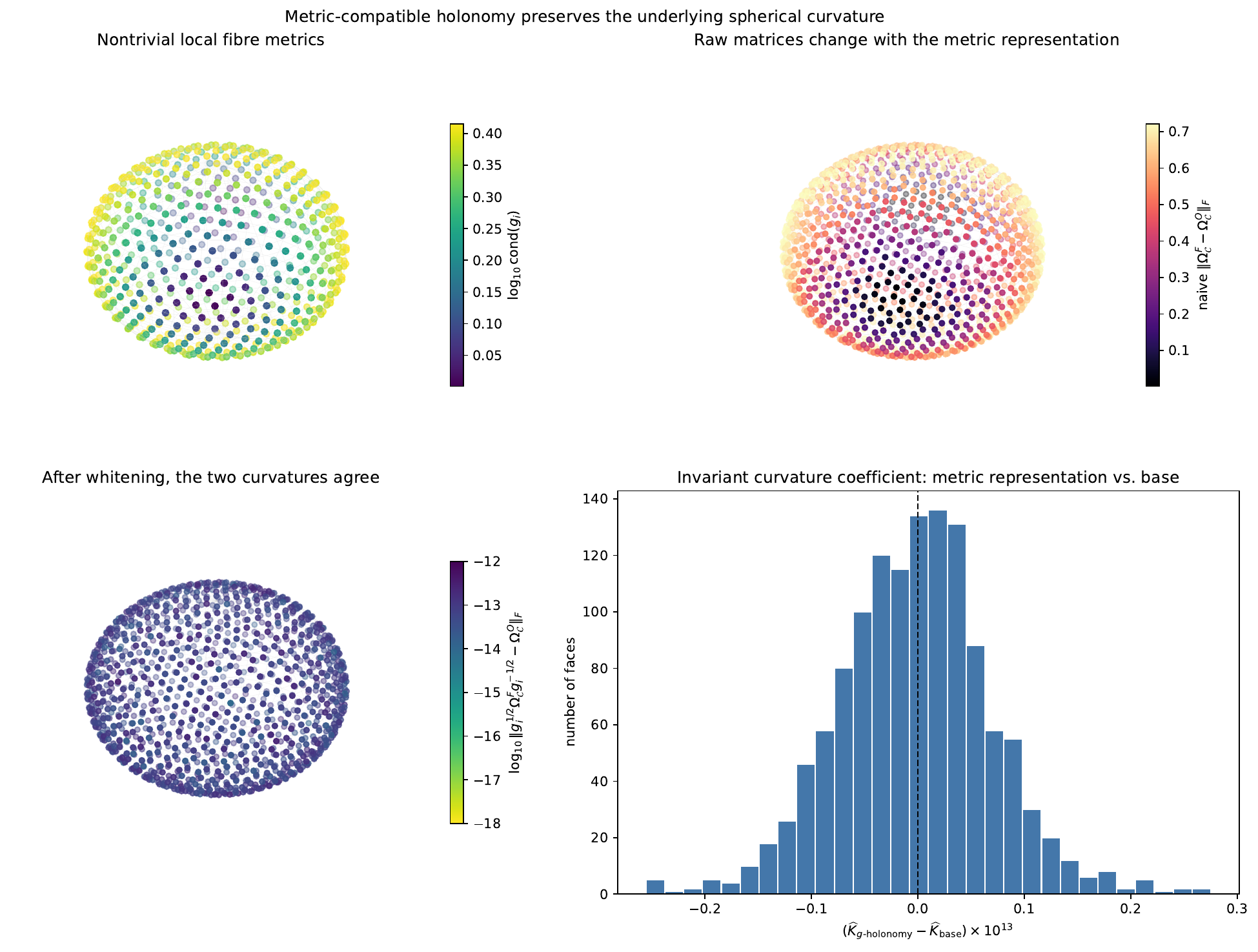}
	\caption{Metric-compatible representation calibration.  A nontrivial
	local SPD field changes the raw matrix representation of loop curvature,
	but the whitened curvature
	\(g_i^{1/2}\Omega_{\cC}^{F}g_i^{-1/2}\) agrees with the orthonormal-frame
	sphere curvature \(\Omega_{\cC}^{O}\) to numerical precision.}
	\label{fig:metric-representation}
\end{figure}

\subsection{Synthetic recovery of learned edge transports}
\label{sec:learned-transport-calibration}

The previous two checks use the exact sphere transport and hence do not test
the learnable degree of freedom in Section~\ref{sec:learning}.  We therefore
construct a small synthetic recovery problem on a fixed level-two icosphere
(\(162\) vertices, \(480\) undirected edges, and \(320\) oriented faces).
For each directed edge \(i\to j\), the estimator observes \(M\) noisy local
vector correspondences
\[
	x_{ij}^{(m)}\in\mathbb R^2,
	\qquad
	y_{ji}^{(m)}=O_{ji}^{\star}x_{ij}^{(m)}
	+\eta_{ij}^{(m)},\qquad m=1,\ldots,M,
\]
where \(O_{ji}^{\star}\) is the hidden exact sphere transport and the noise
standard deviation is \(0.15\).  The hidden transport is used only to
generate the synthetic observations and to evaluate the result.  It is not an
input to the estimator.

For each edge, we use the orthogonal Procrustes rule
\[
	\widehat O_{ji}
	=\underset{O\in\mathrm{SO}(2)}{\arg\min}\,
	\sum_{m=1}^{M}\|Ox_{ij}^{(m)}-y_{ji}^{(m)}\|_2^2,
	\qquad
	\widehat O_{ij}=\widehat O_{ji}^{\mathsf T},
\]
followed by the metric-compatible parameterization
\(\widehat F_{ij}=g_i^{-1/2}\widehat O_{ij}g_j^{1/2}\).  This is one
explicit, CPU-scale realization of the orthogonal-factor learning option in
\eqref{eq:learn-transport}; it is not claimed to be the unique learning
architecture of the paper.  We use the same nontrivial metric field as above,
set \(s_{\cC}=A_{\cC}\), and average over eight independent noise
realizations.

\begin{table}[H]
	\centering
	\caption{Synthetic learned-transport calibration on the fixed sphere mesh.
	Each entry is the mean over eight noise realizations.  The curvature error
	is measured after whitening at the loop base point:
	\(\|\widetilde\Omega_{\cC}^{F}-\Omega_{\cC}^{\star}\|_F\).}
	\label{tab:learned-transport}
	\small
	\begin{tabular}{@{}rrrrrr@{}}
		\toprule
		Correspondences per edge \(M\) & \(1\) & \(4\) & \(16\) & \(64\) & \(256\) \\
		\midrule
		Mean transport error
		\(\|\widehat O-O^\star\|_F\)
		& \(2.094{\times}10^{-1}\) & \(6.683{\times}10^{-2}\)
		& \(3.123{\times}10^{-2}\) & \(1.524{\times}10^{-2}\)
		& \(7.488{\times}10^{-3}\) \\
		Mean whitened curvature error
		& \(1.125{\times}10^{1}\) & \(3.024\)
		& \(1.404\) & \(6.883{\times}10^{-1}\)
		& \(3.437{\times}10^{-1}\) \\
		\bottomrule
	\end{tabular}
\end{table}

\begin{figure}[H]
	\centering
	\includegraphics[width=\textwidth]{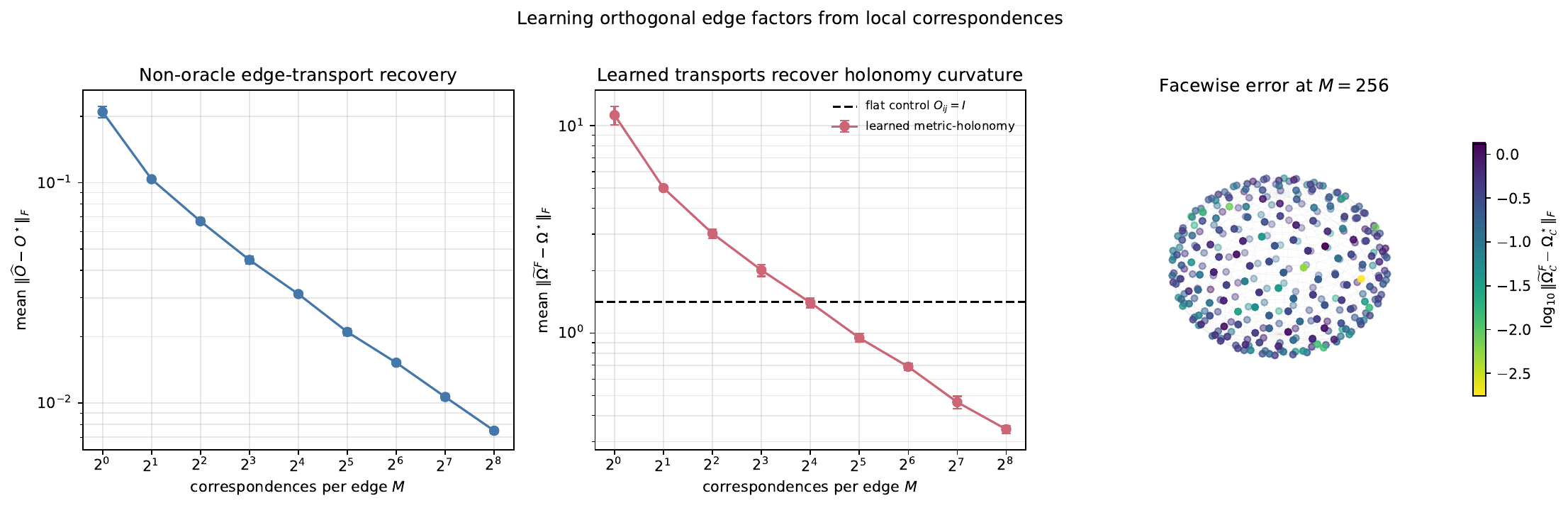}
	\caption{Synthetic recovery of non-oracle edge transports.  Per-edge
		\(\mathrm{SO}(2)\) Procrustes estimates use only noisy vector
		correspondences.  As the number \(M\) of correspondences increases, both
		the learned transport error and the induced whitened holonomy-curvature
		error decrease.  The dashed line is the flat connection control
		\(O_{ij}=I\).}
	\label{fig:learned-transport}
\end{figure}

Table~\ref{tab:learned-transport} and
Figure~\ref{fig:learned-transport} show that, under this observation model,
both the edge transport error and the induced whitened curvature error decline
as \(M\) increases.  At \(M=256\), the mean curvature error is
\(3.44\times10^{-1}\), compared with \(1.41\) for the flat control
\(O_{ij}=I\).  The metric-representation residual remains at numerical
precision throughout, so the observed error is attributable to imperfect
transport recovery rather than to the metric-compatible parameterization.

This is an empirical fixed-mesh consistency trend as the number of
correspondences increases.  It does not establish a limit as the mesh size
tends to zero, nor does it show that an analogous training signal is available
in an arbitrary application.  In particular, a mesh-refinement theorem would
require explicit control of the learned edge error relative to the face scale.

\paragraph{\textbf{Scope of the synthetic checks.}}
The exact algebraic properties proved in the preceding sections---metric
compatibility, reversibility, telescoping of the loop holonomy, and local
orthogonal gauge equivariance---are not inferred from the numerical plots.
They hold by construction under their stated assumptions.  The present checks
instead verify an implementation on a geometry with known curvature and
exhibit one concrete route by which observational information can determine
the otherwise free orthogonal factors.  A subsequent task-level experiment
must separately test whether the resulting geometric information is useful in
data analysis.

\section{Discussion, scope, and next steps}\label{sec:discussion}

The construction proposed here turns a basic fact of connection geometry into
a graph modelling device: nontrivial loop holonomy supplies a local curvature
observation.  Metric compatibility locates this observation in a
$g_i$-orthogonal Lie algebra, and the commutator response converts its
lowered-index antisymmetry into a symmetric metric direction.  This produces a
well-defined SPD-preserving update and makes the role of every additional
choice visible.

Several limitations delimit the current result.  First, triangles are used as
minimal loops, but many graphs have few or no triangles.  Longer cycles or a
chosen cell-complex completion would be needed in sparse settings; their scale
normalization cannot be inferred from the present construction alone.  Second,
the logarithm requires a branch choice and can become numerically unstable near
the excluded spectrum.  Third, the response matrix $B$ is a modelling choice,
not a replacement for the smooth Ricci contraction.  Finally, the calibrations
in Section~\ref{sec:synthetic} do not provide a general continuum limit, a
stability theorem for learned dynamics, or a predictive advantage on a learning
task.  These questions require separate hypotheses and validation.

There are several concrete directions for completing the theory and evaluating
the model.  A cell-complex version could provide an explicit discrete area and
orientation calculus.  A typed tensor formulation may extend the present
orthogonal gauge covariance to broader changes of local frame.  On the learning
side, one should compare fixed versus learned response matrices, test the
effect of loop scales, and evaluate whether curvature-conditioned diffusion
improves task performance without destabilizing metric evolution.  These are
natural continuations of the graph-geometric program, but they are not
assumed in the statements of this paper.

\section*{Acknowledgements}

\end{document}